\documentclass[10pt,a4paper]{scrartcl}
\usepackage{etoolbox}

\newtoggle{release}

\toggletrue{release}

\usepackage[utf8]{inputenc}
\usepackage{lmodern}
\usepackage[T1]{fontenc}
\usepackage{xcolor}
\usepackage{amsmath, amssymb, amsthm}
\usepackage{enumerate}
\usepackage{wrapfig}
\usepackage{hyperref}
\usepackage{url}
\usepackage{pgf}
\usepackage{tikz}
\usepackage{tikzscale}
\usepackage{pgfplots}

\pgfplotsset{compat=newest}
\usepgfplotslibrary{external}

\nottoggle{release}{
\newcommand{\note}[1]{{\textbf{\color{red}#1}}}
\newcommand{\noteal}[1]{{\textbf{\color{blue}#1}}}
}{
\newcommand{\note}[1]{}
\newcommand{\noteal}[1]{}
}

\usetikzlibrary{calc,positioning}
\usepackage{subcaption}
\usepackage{caption}

\usepackage{graphicx}
\usepackage{epstopdf}
\providecommand{\abs}[1]{\left|#1\right|}
\providecommand{\norm}[1]{\left\|#1\right\|}
\providecommand{\normop}[1]{\left\|#1\right\|_{\text{op}}}
\providecommand{\esp}[1]{\mathbb{E}\left[#1\right]}

\providecommand{\prob}[1]{\mathbb{P}\left\{#1\right\}}
\providecommand{\esps}[2]{\mathbb{E}_{#1}\left[#2\right]}

\providecommand{\Tr}[1]{\mathrm{Tr}\left(#1\right)}

\providecommand{\reel}{\mathbb{R}}
\providecommand{\dd}{\mathrm{d}}

\providecommand{\eps}{\varepsilon}

\providecommand{\dotp}[1]{\langle#1\rangle}

\numberwithin{equation}{section}

\providecommand{\gmax}{\gamma_{\mathrm{max}}}

\providecommand{\mms}{\mathcal{I}}
\providecommand{\syms}{\mathcal{S}(\reel^d)}
\providecommand{\ms}{\mathcal{M}(\reel^d)}
\providecommand{\as}{\ \text{a.s}}

\newtheorem{thm}{Theorem}
\newtheorem{lemma}{Lemma}

\newlength{\figwidth}
\newlength{\largefigwidth}
\newlength{\captionsep}

    \newcommand{\ansection}[1]{\subsection*{#1}}
    \newcommand{\ansubsection}[1]{\subsubsection*{#1}}
    \newcommand{\supmat}{Appendix}

\usepackage{natbib}

\usepackage{geometry}
\author{\textbf{Alexandre Défossez} \\
D\'epartement de Math\'ematiques\\
\'Ecole Normale Sup\'erieure \\
 Paris, France \\
\texttt{alexandre.defossez@ens.fr} \\
 \and  \textbf{Francis Bach} \\
D\'epartement d'Informatique\\
\'Ecole Normale Sup\'erieure \\
 Paris, France \\
\texttt{francis.bach@inria.fr} \\
}
\title{Constant Step Size Least-Mean-Square: Bias-Variance Trade-offs and Optimal Sampling Distributions}

\begin{document}
\maketitle

\begin{abstract}

We consider the least-squares regression problem and provide a detailed asymptotic analysis of the performance of averaged constant-step-size stochastic gradient descent (a.k.a.~least-mean-squares). In the strongly-convex case, we provide an asymptotic expansion up to explicit exponentially decaying terms. Our analysis leads to new insights into stochastic approximation algorithms: (a) it gives a tighter bound on the allowed step-size; (b) the generalization error may be divided into a variance term which is decaying as $O(1/n)$, independently of the step-size $\gamma$, and a bias term that decays as $O(1/\gamma^2 n^2)$; (c) when allowing non-uniform sampling, the choice of a good sampling density depends on whether the variance or bias terms dominate. In particular, when the variance term dominates, optimal sampling densities do not lead to much gain, while when the bias term dominates, we can choose larger step-sizes that leads to   significant improvements.
\end{abstract}

\section{Introduction}

\vspace*{-.21cm}

For large-scale supervised machine learning problems, optimization methods based on stochastic gradient descent (SGD) lead to efficient algorithms that make a single or few passes over the data~\citep{ASMB:ASMB538,NIPS2007_3323}.

In recent years, for smooth problems, large step-sizes together with some form of averaging,
 have emerged as having optimal scaling in terms of number of examples, both with asymptotic~\citep{Polyak} and non-asymptotic~\citep{bach11} results. However, these convergence rates in $O(1/n)$ are only optimal in the limit of large samples, and in practice where the asymptotic regime may not be reached, notably because of the high-dimensionality of the data, other non-dominant terms may come into play, which is the main  question we are tackling in this paper.

We consider least-squares regression with constant-step-size stochastic gradient descent---a.k.a.~least-mean-squares---\citep{macchi1995adaptive,bach13}, where the generalization error may be explicitly split into a \emph{bias} term that characterizes how fast initial conditions are forgotten, and a \emph{variance} term that is only impacted by the noise present in the prediction problem. In this paper, we first show that while the variance term is asymptotically dominant, the bias term may play a strong role, both in theory and in practice, that explains convergence behaviors typically seen in applications.

Another question that has emerged as important to improve convergence is the use of special sampling distributions~\citep{nesterov2012efficiency,needell2013stochastic,Zhao2014}. With our theoretical result, we can optimize the first-order asymptotic terms rather than traditional upper-bounds, casting a new light on the potential gains (or lack thereof) of such different sampling distributions.

More precisely, we make the following contributions:

\vspace*{-.25cm}

\begin{list}{\labelitemi}{\leftmargin=1.1em}
   \addtolength{\itemsep}{-.1\baselineskip}

\item[--]
We provide in Section~\ref{sec:main} a detailed asymptotic analysis of the performance of averaged constant-step-size SGD, with all terms up to exponentially decaying ones. We also give in Section~\ref{sec:step} a tighter bound on the allowed step-size $\gamma$.

\item[--]
In Section~\ref{sec:comparing}, the generalization error may be divided into a variance term which is (up to first order) decaying as $O(1/n)$, independently of the step-size $\gamma$, and a bias term that decays as $O(1/\gamma^2 n^2)$.

\item[--] When allowing non-uniform sampling, the choice of a good sampling density depends on whether the variance or bias terms dominate. In particular, as shown in Section~\ref{sec:sampling}, when the variance term dominates, optimal sampling densities do not lead to much gain, while when the bias term dominates, we can choose larger step-sizes that leads to   significant improvements.

\end{list}

\subsection{Problem setup}
\label{sec:setup}

\vspace*{-.2cm}

Let $X$ be a random variable with values in $  \reel^d$ and $Y$ another random variable with values in $ \reel$. Throughout this paper, $\| \cdot \|$ denotes the Euclidean norm on~$\reel^d$. We assume that $    \esp{ \norm{X}^2} = \esp{X^T X} $ is finite and we denote by $H = \esp{X X^T}\in \reel^{d\times d}$ the second-order moment  matrix of $X$. Throughout the paper, we assume that $H$ is invertible, or equivalently, in optimization terms, that we are in the strongly convex case \citep[see, e.g.,][]{nesterov2004introductory}. We   denote by $\mu$ the smallest eigenvalue of $H$, so that we have $\mu > 0$. Note that in our asymptotic results, the leading terms do not depend explicitly on $\mu$.

We wish to solve the following  optimization problem:
\begin{align}
\label{eq:opt_pb}
    \min_{w\in\reel^d} \esp{\norm{X^T w - Y}^2},
\end{align}
from a stream of independent and identically distributed samples $(X_i, Y_i)_{i \geqslant 0}$.

For any given $w\in\reel^d$, we   denote by $f(w) = \esp{\norm{X^T w - Y}^2}$ the expected loss;
 we  denote by $w^* \in \reel^d$ the optimal solution (as $H$ is invertible, it is unique), and by $f^* =f(w^*)\in \reel$ the value at the minimum.

This set-up covers two common situations:

\vspace*{-.15cm}

\begin{list}{\labelitemi}{\leftmargin=1.3em}
   \addtolength{\itemsep}{-.1\baselineskip}

\item[(a)]
 \emph{Single pass through the data}, where each observation is seen once and considered as an i.i.d.~sample, which is the context we explicitly study in this paper; note that then, our bounds are on the \emph{testing error}, i.e., on the expected error on unseen data.
 \item[(b)] \emph{Multiple passes through a finite dataset}, where each sample $(X_i,Y_i)$ is selected uniformly at random from the dataset; in this situation, the \emph{training error} is explicitly minimized, a regularizer is often added and our bound corresponds to training errors. Moreover, dedicated algorithms~\citep{sag,sdca} have then better convergence rates than stochastic gradient.
\end{list}

\textbf{Averaged SGD with constant step-size.}
In this paper, we  study the convergence of the algorithm described by~\cite{bach13}, which is averaged stochastic gradient descent with constant step-size, also often referred to as the averaged  least-mean-squares algorithm (LMS)~\citep[see, e.g.,][]{macchi1995adaptive}.

From a starting point $w_0 \in \reel^d$,
at each iteration $i \geq 1$, an i.i.d.~sample of $(X_i, Y_i)$ is obtained and the following recursion is used:
\begin{eqnarray*}
    w_{i} & = &  w_{i-1} - \gamma X_i (X_i^T w_i - Y_i), \\[-.1cm]
    \bar{w}_{i} & = &  \frac{1}{i+1}\sum_{k=0}^i w_k = \frac{1}{i+1} w_i + \frac{i}{i+1} \bar{w}_{i-1},
\end{eqnarray*}
where $\gamma >0$ is a user-defined step-size.
 We denote by $\eps_i = X_i^T w^* - Y_i$ the residual. Note that by definition of~$w^\ast$, $\esp{ \eps_i X_i}=0$. If the vector $X$ includes a constant component (which is common in practice), then this implies that $\eps_i$ and $X_i$ are \emph{uncorrelated}. Note however that in general they are not \emph{independent}, unless the  model with independent homoscedastic noise is well-specified (which implies in particular that $E(Y_i|X_i) = X_i^T w^\ast$).

We   denote by $f_i = f(\bar{w}_i)$ the value at the averaged iterate. When studying the properties of this algorithm it is more convenient to work with the following centered estimates
$$
    \eta_i =  w_i - w^* \ \ \mbox{ and }  \ \
    \bar{\eta}_i = \bar{w}_i - w^*,
$$
for which one immediately gets
$$
    \eta_{i} = (I - \gamma X_i X_i^T) \eta_{i-1} + \gamma \eps_i X_i,
$$
which is the   recursion that we study in this paper.

\subsection{Related work}

\vspace*{-.2cm}

Stochastic gradient methods have been heavily studied. We mention in this section some of the works which are relevant for the present paper.

\textbf{Analysis of stochastic gradient algorithms.}
Since the work of~\cite{nemirovsky1983problem}, it is known that the optimal convergence rate depends in general on the
presence or absence of \emph{strong convexity}, with rates of $O(1/n\mu)$ for $\mu$-strongly convex problems, and $O(1/\sqrt{n})$ for non-strongly convex problems. Recently, for specific smooth situations with the square or logistic loss, these rates can be improved to $O(1/n)$ in both situations~\citep{bach13}. For least-squares, this is achieved with constant-step-size SGD, hence our main focus on this algorithm.

\textbf{Asymptotic analysis of stochastic gradient descent.}
In this paper, we focus on finding asymptotic equivalents of the generalization errors of SGD algorithms (with explicit remainder terms). For decaying step-sizes and general loss functions, this was partially considered by~\citet{fabian1968}, but only without averaging. Moreover, the traditional analysis of Polyak-Rupert averaging~\citep{Polyak,ruppert} also leads to asymptotic equivalents, also for decaying step-sizes, but only the (asymptotically dominant) variance terms are considered.

\textbf{Non-uniform sampling.}
Non-uniform sampling has been already tackled from several points of views; for example, in the active learning literature,
\citet{KANAMORI} provide the optimal sampling density to optimize the generalization error (for an estimator obtained as the minimum of the empirical least-squares risk), leading to distributions that are the same than the one obtained in Section~\ref{sec:optimal-var} (where the variance term dominates), for which the actual gains are limited.

Moreover, in the context of stochastic gradient descent methods, \citet{needell2013stochastic,Zhao2014} show that by optimizing the sampling density, bounds on the convergence rates could be improved, but the actual gains are hard to quantify. Our focus on limits of convergence rates allows us to precisely quantify the gains and obtain extra insights (at least asymptotically).

\section{Linear algebra prerequisites}
\label{sec:algebra}
Throughout our results we will use the following notations and results. These are necessary to provide explicit expressions for the constants in the asymptotic expansions.

For any real vector space $V$ of finite dimension $d$, let $\mathcal{M}(V)$ be the space of linear operators over~$V$
which is isomorphic to the space of $d$-by-$d$ matrices, with the usual results that composition becomes matrix multiplication. As a consequence we will use the same notation for the space of matrices and the space of endomorphisms.

We   denote by $\mms = \mathcal{M}(\mathcal{M}(\reel^d))$ the space of endomorphisms on the space of matrices over $\reel^d$. One can index the rows and columns of a matrix $M\in\mms$ by a pair $(i, j)$ where $1 \leq i\leq d$ and $1 \leq j \leq d$. We will often denote by $M_{(i, j), (k, l)}$ an element of this matrix on matrices.
In the following we will drop the domain of $i, j, k, l, i', j'$ which is implicitly $\{1, 2, \ldots, d\}$.
Explicitly, if  $A \in \ms$ and $M \in \mms$, then $MA$ is defined through:
\begin{align*}
    \forall (i, j) (M A)_{i, j} &= \sum_{i'=1, j' = 1}^{d} M_{(i, j), (i', j')} A_{i', j'}\\
\end{align*}

We will mostly make no distinction between $A$ as a vector in $\ms$ on which elements in $\mms$ can operate and $A$ as a matrix in $\ms$. Then $M A$ can be
either usual matrix multiplication if $M, A \in \ms$ or $M, A\in \mms$ or application of $M$ to $A$ if $M\in \mms$ and $A \in \ms$.
However, if $M \in \mms$ and $A \in \ms$, then $A M$ does not make sense.
For $P \in \mms$, and any $(i, j)$, we will define $P_{i, j}$ the matrix in $\ms$ with coefficient $(i', j')$ given by $P_{(i, j), (i', j')}$.

For any $V \in \syms$ (the set of symmetric matrices of size $d$), we will denote $\normop{V}$ the operator norm of $V$ or equivalently its eigenvalue with the largest absolute value.
For any $M \in \mms$ so that $\syms$ is stable under $M$, we will take $\normop{M}$ the operator norm of $M$ restricted to $\syms$, defined with respect to the Frobenius
 norm on $\syms$, that is
\begin{align*}
    \normop{M} = \sup_{V\in\syms, \norm{V}_F = 1} \norm{M V}_F.
\end{align*}
Equivalentely, it is given by the largest absolute value of the eigenvalues of $M$.

Finally, we will look more precisely at three elements of $\mms$. For any given $A \in \ms$, one can define $A_L$ (resp.~$A_R$) so that $A_L$ is the matrix in $\mms$ representing left multiplication (resp.~right multiplication) by $A$. The coefficient of $A_L$ and $A_R$ are given by
\begin{align*}
    \forall (i, j), (k, l), \ \ (A_L)_{(i, j), (k, l)} &= \delta_{j, l} A_{i, k}\\
    \forall (i, j), (k, l), \ \ (A_R)_{(i, j), (k, l)} &= \delta_{i, k} A_{j, l}.
\end{align*}

If $A$ is symmetric, then $A_L$ and $A_H$ are both symmetric operators, and $A_L + A_H$ is stable on the subspace of symmetric matrices, that we have denoted $\syms$.

Let $X$ be a random variable in $\reel^d$, we consider the linear operator $M$ on $\ms$ defined by,
\begin{align*}
    \forall A\in \ms,\ \  MA = \esp{ (X^T A X) X X^T},
\end{align*}
then, the coefficients of the associated matrix are given by
\begin{align*}
    \forall (i,j,k,l) , \ \ M_{(i, j), (k, l)} = \esp{X^{(i)} X^{(j)} X^{(k)} X^{(l)}},
\end{align*}
where $X^{(i)}$ denote the $i$-th component of the vector $X$. The matrix $M$ is clearly symmetric. One can also prove that it is stable on $\syms$.

We then define $T = H_L + H_R - \gamma M$ with $H_L$, $H_R$ and $M$ as defined above for the random variable $X$ defined in our setup. It is immediately stable over $\syms$. We will denote $\mu_T$ the smallest eigenvalue of $T$.

\section{Main results}
\label{sec:main}

\vspace*{-.21cm}

We will present results about the convergence of the algorithm which are derived from the exact computation of the second-order moment matrix, which we refer to as the \emph{covariance matrix}, $\esp{\bar{\eta}_n \bar{\eta}_n^T}$. Since we consider a least-squares problem, we have
$$f_n - f^\ast = {\rm Tr}  \big( H \esp{\bar{\eta}_n \bar{\eta}_n^T} \big).$$
 We will distinguish two terms, which can be assimilated to a variance/bias decomposition.
The \emph{variance} term $\Delta^{\mathrm{variance}}$
can be defined as the covariance matrix we would get starting from the solution (that is, $\eta_0 = 0$). On the other hand, the \emph{bias} term $\Delta^{\mathrm{bias}}$ is defined as the covariance matrix we would get if the model was noiseless, that is $Y = X^T w^*$ and $\eps=0$.

Each of these two terms leads to contribution to $f_n - f^\ast $, that is
$ {\rm Tr}  \big( H\Delta^{\mathrm{variance}}  \big)$ and
$ {\rm Tr}  \big(H \Delta^{\mathrm{bias}} \big)$.
Under extra assumption that are discussed in the supplementary material, such that when $X$ and $\eps$ are independent (i.e., well-specified model), the actual covariance matrix is exactly the sum of the bias and variance matrices, and thus
$$
f_n - f^\ast = {\rm Tr}  \big( H\Delta^{\mathrm{variance}}  \big) +  {\rm Tr}  \big(  H\Delta^{\mathrm{bias}} \big).
$$
Moreover, even when this is not true, it has been noted by~\cite{bach13} that
$$ f_n - f^\ast \leq 2  {\rm Tr}  \big( H\Delta^{\mathrm{variance}}  \big) + 2  {\rm Tr}  \big( H\Delta^{\mathrm{bias}} \big),
$$
that is, the sum of the two terms is a factor of two away from the exact generalization error.

\subsection{Improved step-size}
\label{sec:improved_stepsize}
\label{sec:step}

\vspace*{-.2cm}

Let us take $T$ as defined in Section~\ref{sec:algebra}.
We will also define two contraction factors,
\begin{align}
    \rho_T = \normop{I - \gamma T}\quad\text{and}\quad
    \rho_H = \normop{I - \gamma H},
\end{align}
as well as $\rho = \max(\rho_T, \rho_H)$, where $\normop{\cdot}$ is defined as the largest eigenvalue in absolute value.

Let us define $\gmax$ as the supremum of the set of $\gamma >0$ verifying, $\forall A\in \syms$ (the set of symmetric matrices of size $d \times d$):
\begin{align}
  \,2 \Tr{A^T H A} - \gamma \esp{(X^TA X)^2} > 0,
\end{align}
or equivalently as the supremum of $\gamma > 0$ such that~$T$ is definite positive.
One can actually show that we necessarily have that
\begin{align}
    \label{eq:majoration_gmax}
    \gmax \leq  {2}/ {\Tr{H}},
\end{align}
and the following lemma (see proof in the \supmat):

\begin{lemma}
    \label{lemma:rho}
    Using the notations and assumptions of Section~\ref{sec:setup},  define $\gmax$ as the supremum of $\gamma > 0$ such that
    \begin{align}
        \label{eq:strong}
        \forall A\in \syms, \ \,2 \Tr{A^T H A} - \gamma \esp{(X^TA X)^2} > 0.
    \end{align} If $0 < \gamma < \gmax$ then $T$ is positive definite and $\rho < 1$. More precisely,
    in dimension $d \geqslant 2$,  we have
    \begin{equation}
        \label{eq:rho}
        \begin{cases}
           \displaystyle \rho \leq 1 - 2 \gamma \left(1 - \frac{\gamma}{\gmax}\right) \mu \quad&\text{if $1 > \frac{\gamma}{\gmax} \geq \frac{1}{2}$}\\
            \rho \leq 1 - \gamma \mu \quad&\text{otherwise}.
        \end{cases}
    \end{equation}
    In dimension $d=1$, we have \[\rho \leq \max\left(\abs{1 - \gamma \mu}, 1 - 2 \gamma\left(1 - \frac{\gamma}{\gmax}\right) \mu  \right).\]

    Otherwise, if $\gamma > \gmax$, then $\rho > 1$.
\end{lemma}

Note that we may rewrite $\gmax$ as
$$
\frac{2}{\gmax} =
\sup_{A \in \syms} \frac{\esp{(X^TA X)^2}}{\Tr{A^T H A} },
$$
which can be computed explicitly by a generalized eigenvalue problem once all second- and fourth-order moments of $X$ are known. This is to be contrasted with the largest step-size $\gmax^{\rm det}$ for deterministic gradient descent, which is such that
$$
\frac{2}{\gmax^{\rm det}} =
\sup_{a \in \reel^d} \frac{ a^T  H a }{ a^T a}.
$$
One can observe that for any distribution on $X$, we necessarily have $\gmax \leq  {2}/{\Tr{H}} \leq \gmax^{\mathrm{det}}$ so that the maximum stochastic step-size will always be smaller than the deterministic, one as one would expect.

Note also that the step-size provided by $\gmax$ is a strict improvement (see supplementary material) on the one proposed by~\citet{bach13}, which is equal to the supremum of the set of $\gamma>0$ such that
$\esp{ X X^T }- \gamma \esp{ (X^T X) X X^T }$ is positive definite.

We conjecture that the bound given by $\gmax$ is tight, namely that if $\gamma$ is larger than $\gmax$ then there exists an initial condition $\eta_0$ such that the algorithm   diverges.

\subsection{Bias term}

\vspace*{-.2cm}

In this section, we provide an asymptotic expansion of the bias term $\Delta^{\mathrm{bias}}$.
\begin{thm}[Asymptotic covariance of the bias term]
    \label{bias_theo}
     Let $E_0 =  \eta_0 \eta_0^T$.
    If $0 < \gamma < \gmax$ and  $\forall i\geq 1, \eps_i = 0$, then
    
    \begin{align}
    \label{eq:result_bias}
    \Delta^{\mathrm{bias}} =   \esp{\bar{\eta}_n\bar{\eta}_n^T} =
            \frac{1}{n^{2}  \gamma^{2} } \left(H_L^{-1} + H_R^{-1}- \gamma I \right)\left(T^{-1} E_0\right) + O\left(\frac{\rho^n}{n}\right).
     \end{align}
        Explicit bounds are given in the proof in the Appendix.

\end{thm}
A detailed proof is given in the \supmat. Using Lemma~\ref{lemma:rho}, we know that $\rho < 1$ so that \eqref{eq:result_bias} converges as $n^{-2}$.
From that we can derive that the rate of convergence for $ {\rm Tr}  \big( H\Delta^{\mathrm{bias}} \big)$, that will be of order $n^{-2}$ as well. Although the dependency of $A(\gamma)$ is complex, one can easily derive an equivalent when $\gamma$ tends to zero, and we have
\begin{align}
\label{eq:approx_bias}
    \lim_{n\to \infty} n^2  {\rm Tr}  \big( H\Delta^{\mathrm{bias}} \big) &\underset{\gamma \to 0}{\sim} \gamma^{-2}  \eta_0^T H^{-1} \eta_0 .
\end{align}

\subsection{Variance term}

\vspace*{-.2cm}

In this section, we provide an asymptotic expansion of the variance term
$\Delta^{\mathrm{variance}}$.

\begin{thm}[Asymptotic covariance of the variance term]
    \label{var_theo}
     Let $\Sigma_0 = \esp{\eps^2 X X^T}$ and let assume that $\eta_0 = 0$.
     If $0 < \gamma < \gmax$ then $\Delta^{\mathrm{variance}}$ is equal to:
     
    \begin{align}
             \esp{\bar{\eta}_n\bar{\eta}_n^T} =
               \frac{1}{n} (H_L^{-1} + H_R^{-1} - \gamma I)T^{-1} \Sigma_0
                 -\frac{1}{\gamma n^2 } \left(
        H_L^{-1} + H_R^{-1} - \gamma I
    \right) (I - \gamma T) T^{-2} \Sigma_0 + O\left(\frac{\rho^n}{n}\right).
    \end{align}
        Explicit bounds are given in the proof in the Appendix.
    
\end{thm}
A detailed proof is given in the \supmat.

Unsurprisingly, the asymptotic behavior of the variance term is dominant over the bias one as it decreases only as $n^{-1}$, which is the overall convergence rate of this algorithm of least-mean-squares as noted by \cite{bach13}.

It is also possible to get a simpler equivalent when $\gamma$ goes to $0$:
\begin{align*}
     \lim_{n\to \infty} n   {\rm Tr}  \big( H\Delta^{\mathrm{variance}} \big) &\underset{\gamma \to 0}{\sim} \esp{\eps^2 X^T H^{-1} X}.
\end{align*}

If we further assume that the noise $\eps$ is independent of~$X$, then we recover the usual
\begin{align*}
    \lim_{n\to \infty} n   {\rm Tr}  \big( H\Delta^{\mathrm{variance}} \big)  &\underset{\gamma \sim 0}{\sim} d \sigma^2,
\end{align*}
where $\sigma = \esp{\eps^2}$, which is the Cramer-Rao bound for such a problem.
It is also interesting to notice that this is the exact same result as the one obtained by~\citet{Polyak} with a decreasing step-size.

Finally, note that if $\gamma$ is small, the term in $\frac{1}{\gamma n^2 } C(\gamma) \Sigma_0$ is always positive so that there is no risk of it exploding for small values of $\gamma$ (unlike for the bias term).

\subsection{Comparing both terms}

\vspace*{-.2cm}

\label{sec:comparing}

As seen above with an asymptotic expansion around $\gamma = 0$, for $n$ sufficiently large,
the bias and variance terms are of order:
\begin{eqnarray*}
 {\rm Tr}  \big( H\Delta^{\mathrm{bias}} \big) &\sim & \frac{1}{\gamma^2 n^2} \eta_0^T H^{-1} \eta_0
\\
 {\rm Tr}  \big( H\Delta^{\mathrm{variance}} \big) &\sim &  \frac{1}{n} \esp{\eps^2 X^T H^{-1} X}.
\end{eqnarray*}
The different behaviors of the bias and variance terms lead to two regimes, one in $(\gamma n)^{-2}$ and one in $n^{-1}$ that can clearly be observed on synthetic data. On real world data, one will often observe a mixture of the two, depending on the step-size and the difficulty of the problem. Experimental results on both synthetic and real world data will be presented in Section~\ref{sec:exp}.

\section{Optimal sampling}

\label{sec:sampling}

\vspace*{-.21cm}

Changing the sampling density may be interesting in several situations, in particular (a) in presence of outliers   (i.e., points with large norms) and (b) classification problems with asymmetric costs (see Section~\ref{sec:asym}).

\subsection{Impact of sampling}

\vspace*{-.2cm}

Using the two previous theorems, we can now try to optimize the sampling distribution   to increase performance. We will sample from a distribution $q$ instead of the given distribution $p$. Since we wish to keep the same objective function, we will use importance weights $c(X,Y)$, so that if we  denote by $\esps{p}{A}$ the expectation of a random variable $A$ under the probability distribution given by $p$ over $A$ we have
$$
    \esps{p}{|X^Tw - Y|^2}  =
        \esps{q}{c(X, Y)|X^T w - Y|^2}.
$$
First, one can notice that, from a practical point of view, we must restrict ourselves to $q$ that are absolutely continuous with respect to $p$ as one cannot invent samples. In order to be able to define $c$ we also need   $p$ to be absolutely continuous with respect to $q$, so that \[c = \frac{\dd p}{\dd q}.\] Besides, $c^{-1}$ is   defined as $c^{-1} = \frac{\dd q}{\dd p}$.

A key consequence of using least-squares is that for a given $(q, c)$ pair, we only have to sample using $q$ and scale $X$, $Y$ and $\eps = X^T w^* - Y$ by $\sqrt{c(X, Y)}$. Thus we can use the two previous theorems for $X' = \sqrt{c(X, Y)} X$ and $Y' = \sqrt{c(X, Y)} Y'$ and sampling $X, Y$ according to $q$.

As of now, we will assume that almost surely $X \neq 0$. Indeed, when $X_i = 0$, we perform no update so that we can just ignore such points.

One can notice that for any $A, B\in\{X, Y, \eps\}$,
\begin{align*}
    \esps{q}{A' B'} &= \esps{q}{c(X, Y) A B} = \esps{p}{A B},
\end{align*}
and thus all second-order moments are unchanged  under resampling. This is the case for the matrix $H = \esps{p}{X X^T} = \esps{q}{ X' X'^T}$ for instance.

However, for terms of order 4 like $T$, $\esp{ (X^T X) X X^T}$ or $\Sigma_0$, an extra $c$ appears and we have for instance
$$
\esps{q}{(X'^T X') X' X'^T} = \esps{p}{\frac{p(X, Y)}{q(X, Y)} (X^T X) X X^T}.
$$
It means that while $H$ will not be changed, $T$ is impacted in non trivial ways, as is $T^{-1}$. This makes it tricky to truely optimize sampling for any $\gamma$. However, when assuming $\gamma$ small, it is possible to optimize the limit we obtained for $\gamma \to 0$ (see Sections~\ref{sec:optimal-var} and~\ref{sec:optimal-bias}). The experiments we ran (see Section \ref{sec:exp}) seem to confirm that this is a valid assumption for values of $\gamma$ as high as $ {\gmax}/{2}$.

\subsection{Asymmetric binary classification}

\vspace*{-.2cm}

\label{sec:asym}
As a motivation for this work, we will present one practical application of resampling which is binary classification with highly asymmetric classes.
Assume we have $Y \in \{-1, 1\}$ and that $\prob{Y=1}$ and $\prob{Y=-1}$ are highly unbalanced, as it can be the case in various domains, such as ad click prediction or object detection, etc. Then, it can be useful in practice to give more weight the less frequent class~\citep[see, e.g.,][and references therein]{perronnin2012towards}. This is equivalent to multiplying both $X$ and $Y$ by some constant $\sqrt{c_Y}$.
A common choice is for instance to take $c_y =  {1}/ {\prob{Y=y}}$ which will give the same importance in the loss to both classes.

However, these weights will make the gradients from the less frequent class huge compared to the usual updates. This is likely to impact the convergence of the algorithm. In that case it is  easy to notice that taking taking $c(x, y) =  {1}/{c_y}$ will leave the gradients unchanged but will favor sampling examples from the less frequent class.

\subsection{Optimal sampling for the variance term}
\label{sec:optimal-var}
Let assume that we are only interested in the long term performance for our algorithm. Ultimately the variance term will be driving the performance and we need to optimize it.

Exactly optimizing the sampling for this case in uneasy as it impacts both $\Sigma_0$ and the terms $B(\gamma)$ and $C(\gamma)$ in Theorem~\ref{var_theo},  in a non trivial way. However, if we assume a small step-size $\gamma$, then we just have to minimize
\begin{align*}
    \esps{q}{\eps'^2 X'^T H^{-1} X'} &=
        \esps{p}{c(X, Y)\eps^2 X^T H^{-1} X},
\end{align*}
under the constraint that $\esps{p}{c^{-1}(X, Y)} = 1$ so that~$q$ is a distribution.
Using the Cauchy-Schwarz inequality, we have that

\begin{align*}
    \esps{p}{c(X, Y)\eps^2 X^T H^{-1} X} =&
        \esps{p}{c(X, Y)\eps^2 X^T H^{-1} X} \esps{p}{c^{-1}(X, Y)}\\
        &\geq \big(
        \esps{p}{\abs{\eps}\sqrt{X^T H^{-1} X}} \big)^2.
\end{align*}

When $X \neq 0$ almost surely, then this lower-bound is achieved for
\[
    c^{-1}(X, Y) = \frac{\abs{\eps}\sqrt{X^T H^{-1} X}}{\esps{p}{\abs{\eps}\sqrt{X^T H^{-1} X}}},
\]
which requires prior knowledge of $H$ and $\eps$.

In that case, we obtain
\begin{align*}
    \lim_{n \to \infty} n   {\rm Tr}  \big( H\Delta^{\mathrm{variance}} \big) &= \left(\esp{\abs{\eps} \sqrt{X^T H^{-1} X}}\right)^2.
\end{align*}
One can notice that this is the exact same optimal sampling as the one obtained in the active learning set-up by~\cite{KANAMORI}.

Again, it is possible to slightly simplify this expression when $\eps$ and $X$ are independent, as we obtain
\begin{align*}
    \lim_{n \to \infty} n  {\rm Tr}  \big( H\Delta^{\mathrm{variance}} \big) &= \sigma^2 \left(\esp{\sqrt{X^T H^{-1} X}}\right)^2,
\end{align*}
with $\sigma^2 = \esp{\eps^2}$.
At this point it is important to realize that the gain we have here is of the order of
\begin{align}
    \label{eq:gain_var}
     {\esp{\sqrt{X^T X}}^2}/ \ {\esp{X^T X}}.
\end{align}
During our experimentations on usual datasets, we have observed that this factor was always between $1/2$ and $1$ and thus there is   little to be gained when optimizing the variance term.

\subsection{Optimal sampling for the bias term}
\label{sec:optimal-bias}
Although asymptotically the variance term will be the largest one, it is possible that initially the bias one is non negligible and it can be interesting to optimize for it. This is all the more possible as it depends much more on the step-size $\gamma$ and if $\gamma$ is too small, the bias term can stay larger than the variance term for many iterations.

If we assume $\gamma$ small, then we can approximate the bias term by the expression given by \eqref{eq:approx_bias}, that is, proportional to $1/(\gamma^2 n^2)$. In this case, it is clear that we want to increase $\gmax$ and, because second-order moments are not impacted by resampling, it has no effect other than changing $\gmax$.
Numerical experiments tends to show that increasing $\gamma$ is beneficial even for $\gamma$ close to $\frac{\gmax}{2}$. Beyond this limit, the approximation \eqref{eq:approx_bias} is no longer sustainable and besides, exponentially decreasing terms can start to grow quite large.

The maximum step-size we can take is given by the tighter condition from Section~\ref{sec:step}, that is, $\forall A\in \syms$,
\begin{align}
\label{condition}
    \,2 \Tr{A^T H A} - \gamma \esp{(X^TA X)^2} > 0,
\end{align}
which implies that
\begin{align}
\label{eq:opt_gamma}
    \frac{2}{\esp{X^T X}} \geq \gmax,
\end{align}
using \eqref{eq:majoration_gmax}. As this upper bound on $\gmax$ only depends on moments of order 2, and that those moments are not changed by resampling, \eqref{eq:opt_gamma} is an upper bound on any $\gmax$ for a given optimization problem, no matter how we resample. It turns out it can be achieved by  the resampling given by
\[
    c_*^{-1}(X, Y) = \frac{X^T X}{\esps{p}{X^T X}} \as,
\]
which, unlike the variance term, does not require the knowledge of $H$. We have

\begin{align*}
    H - \gamma \esps{q}{c_*(X, Y)^2 X^T X X X^T} &=
    2H - \gamma \esps{p}{\frac{\esps{p}{X^T X}}{X^T X} (X^T X) X X^T}\\
    &= H (2 - \gamma \esps{p}{X^T X}),
\end{align*}
which is positive definite  as soon as $\gamma < \frac{2}{\Tr{H}}$. Besides, one can prove that $H - \gamma \esp{(X^T X) X X^T} \succ 0$ is a stronger form of \eqref{condition} and implies it, as we already noted in Section~\ref{sec:improved_stepsize}.
This means that using the resampling defined by $c_*$, we have $\gmax = \frac{2}{\esp{X^T X}}$ which is thus not improvable.

If $\gmax^{(0)}$ is the maximum step-size before resampling and $\gmax^{(1)}$ is the maximum step-size after resampling, then the gain for $f_n - f^*$ is a factor
$
    \big( {\gmax^{(0)}}/{\gmax^{(1)}}\big)^2.
$
It can be hard to evaluate, but from our experiments (see Section \ref{sec:exp}) it was common to observe gain factor of $ {1}/{100}$ or $ {1}/{400}$ while the gain for the variance term was limited to ${1}/{2}$.

Unlike the variance term, the resampling in itself here has an impact only through a  larger step-size. Resampling while keeping the same step-size will often lead to almost identical performances for the bias term. It is interesting to note that when $H = I$, this sampling will exactly have no impact at all on the variance term, and when $H\neq I$, it will only impact it marginally.

\begin{figure}[t]
    \centering
    \begin{minipage}{.49\textwidth}
    \centering
    \begin{tikzpicture}
        \begin{loglogaxis}[
            xlabel={Iteration $n$},
            ylabel=$f_n - f^*$,
            x label style={at={(axis description cs:0.5,-0.08)},anchor=north},
            y label style={at={(axis description cs:-0.11,.5)},rotate=0,anchor=south},
            legend entries={{$q(X) \propto X^T X$, step=0.05}, {$q(X) \propto X^T X$, step=0.115}, {uniform, step=0.05}},
            legend style={legend pos = south west},
            style={font=\footnotesize},
            width=\figwidth]
            \pgfplotstableread{PAPER_yahoo.txt}\tableyahoo
            \addplot+[mark=o] table[x=n, y=opt_small] {\tableyahoo};
            \addplot+[mark=+] table[x=n, y=opt_large] {\tableyahoo};
            \addplot+[mark=triangle] table[x=n, y=small] {\tableyahoo};
        \end{loglogaxis}
    \end{tikzpicture}
   \vspace{-.5\captionsep}
    \captionof{figure}{Convergence on \emph{Yahoo} dataset without weights.}
     \vspace{.25\captionsep}
    \label{fig:yahoo_noweights}
    \end{minipage}
    \begin{minipage}{.49\textwidth}
    \centering
    \begin{tikzpicture}
        \begin{loglogaxis}[
            xlabel={Iteration $n$},
            ylabel=$f_n - f^*$,
            x label style={at={(axis description cs:0.5,-0.08)},anchor=north},
            y label style={at={(axis description cs:-0.11,.5)},rotate=0,anchor=south},
            legend entries={{$q(X) \propto X^T X$, step=0.005}, {$q(X) \propto X^T X$, step=0.105}, {uniform, step=0.005}},
            legend style={legend pos = south west},
            style={font=\footnotesize},
            width=\figwidth]
            \pgfplotstableread{PAPER_yahoo_weights.txt}\tableyahoo
            \addplot+[mark=o] table[x=n, y=opt_small] {\tableyahoo};
            \addplot+[mark=+] table[x=n, y=opt_large] {\tableyahoo};
            \addplot+[mark=triangle] table[x=n, y=small] {\tableyahoo};
        \end{loglogaxis}
    \end{tikzpicture}
    \vspace{-.5\captionsep}
    \captionof{figure}{Convergence on \emph{Yahoo} dataset with weights.}
     \vspace{.25\captionsep}
      \label{fig:yahoo_weights}
    \end{minipage}
\end{figure}

\begin{figure}[t]
    \centering
    \begin{minipage}{.49\textwidth}
    \centering
    \begin{tikzpicture}
        \begin{loglogaxis}[
            xlabel={Iteration $n$},
            ylabel=$f_n - f^*$,
            x label style={at={(axis description cs:0.5,-0.08)},anchor=north},
            y label style={at={(axis description cs:-0.11,.5)},rotate=0,anchor=south},
            style={font=\footnotesize},
            legend entries={{$q(X) \propto X^T X$, step=0.00197}, {$q(X) \propto X^T X$, step=0.00206}, {uniform, step=0.00197}},
            legend style={legend pos = south west},
            width=\figwidth]
            \pgfplotstableread{PAPER_sido.txt}\tableyahoo
            \addplot+[mark=o] table[x=n, y=opt_small] {\tableyahoo};
            \addplot+[mark=+] table[x=n, y=opt_large] {\tableyahoo};
            \addplot+[mark=triangle] table[x=n, y=small] {\tableyahoo};
        \end{loglogaxis}
    \end{tikzpicture}
   \vspace{-.5\captionsep}
    \captionof{figure}{Convergence on \emph{Sido} dataset without weights.}
   \vspace{.25\captionsep}
      \label{fig:sido_noweights}
    \end{minipage}
    \begin{minipage}{.49\textwidth}
    \centering
    \begin{tikzpicture}
        \begin{loglogaxis}[
            xlabel={Iteration $n$},
            ylabel=$f_n - f^*$,
            x label style={at={(axis description cs:0.5,-0.08)},anchor=north},
            y label style={at={(axis description cs:-0.11,.5)},rotate=0,anchor=south},
            legend entries={{$q(X) \propto X^T X$, step=0.00013}, {$q(X) \propto X^T X$, step=0.0017}, {uniform, step=0.00013}},
            legend style={legend pos = south west},
            style={font=\footnotesize},
            width=\figwidth]
            \pgfplotstableread{PAPER_sido_weights.txt}\tableyahoo
            \addplot+[mark=o] table[x=n, y=opt_small] {\tableyahoo};
            \addplot+[mark=+] table[x=n, y=opt_large] {\tableyahoo};
            \addplot+[mark=triangle] table[x=n, y=small] {\tableyahoo};
        \end{loglogaxis}
    \end{tikzpicture}
    \vspace{-.5\captionsep}
    \captionof{figure}{Convergence on \emph{Sido} dataset with weights.}
      \vspace{.025\captionsep}
    \label{fig:sido_weights}
    \end{minipage}
\end{figure}

\textbf{Link with other algorithms.}
When using this resampling and $\gamma =  {1}/ {\esp{X^TX}}$, the update step becomes
\begin{align}
   w_{i}     \label{eq:our_update}
            &=  w_{i-1} -  \frac{1}{X_i^T X_i} \left(X_i^T w_{i-1} - Y_i\right),
\end{align}
where $X_i$ is sampled from $q_*$.
This is very similar to normalized least mean squares (NLMS) by \cite{normalized}, i.e., we first normalize $X$ (and $Y$ by the same factor), and then we run the usual stochastic gradient descent with a step-size of 1.
However, while \textsc{NLMS} does not optimize the same overall objective function, we remember the norm of $X$ in $c_*$ and sample large ones more often and keep the same overall objective function. One can also notice some links with implicit stochastic gradient descent (ISGD) by \cite{implicit}, where the update rule is
\begin{align}
\label{eq:imp_update}
    w_{i} &= w_{i-1}- \frac{\gamma_i}{1+ \gamma_i X_i^TX_i} \left(X_i^T w_{i-1} - Y_i\right),
\end{align}
which is similar to \textsc{NLMS} and \eqref{eq:our_update} when $\gamma_i$ is large. As $\gamma_i$ is a decreasing step-size it means that during the early iterations, \textsc{ISGD} will behave like \textsc{NLMS} before switching to a regular stochastic gradient descent as $\gamma_i$ goes to 0. This comforts us in the idea that a large step-size is crucial during the early stages, as highlighted by our analysis.

\section{Experiments}

\vspace*{-.21cm}

\label{sec:exp}

\textbf{\emph{Yahoo} and \emph{Sido} datasets.}
We have tried to observe evidence of our predictions into two unbalanced datasets, ``yahoo'' and ``sido''.
The ``yahoo'' dataset\footnote{\url{webscope.sandbox.yahoo.com/}} is composed of millions of triple (ads, context, click) from the Yahoo front page where click is 1 if the user clicked on the given ad in the given context (composed of 136 Boolean features) and 0 otherwise. It is very unbalanced, as the click rate is  low. For this experiment we only looked at the rows corresponding to a specific ad (107355 rows), with a click rate of 0.03\%. We used it both with and without weights in order to give the same importance to both clicks and non clicks as explained in Section \ref{sec:asym}.

On Figures~\ref{fig:yahoo_noweights} and \ref{fig:yahoo_weights}, we compared the incidence of the step-size and the resampling of performance. When using weights, the maximum step-size happens to be divided by ten, but when resampling proportionally to $X^T X$, we recover almost the same step-size as without weights.
Moreover, without weights (Figure~\ref{fig:yahoo_noweights}), the lowest step-size performs best, which tends to indicate that the variance term is dominant. However, with weights (Figure~\ref{fig:yahoo_weights}), one can observe the more pronounced dependency in $\gamma$, which show that the bias term became non negligible.
Comparable results, but obtained on the ``sido'' dataset\footnote{\url{www.causality.inf.ethz.ch/data/SIDO.html}} are given in Figures~\ref{fig:sido_noweights} and \ref{fig:sido_weights}. This dataset is composed of 12678 points with 4932 features. The less frequent class represents 3.6\% of the points.

 We did not plot the graphs obtained when sampling proportionally to $\sqrt{X^T H^{-1} X}$ as they are mostly the same as without resampling. If we were to achieve the regime where the variance is completely dominant, the error would be at best be divided by two for the ``yahoo'' dataset and would be almost the same for ``sido'' (divided by 1.02). These potential gains were computed directly using the expression \eqref{eq:gain_var}.

\textbf{Synthetic data and bias-variance decomposition.}
We also observed exactly the bias and variance terms on synthetic data. The data consist in an infinite stream of points $X$ sampled from a normal distribution with covariance matrix $H$, so that the eigenvalues of $H$ are $(\frac{1}{i})_{1 \leq i\leq 25}$. $Y$ is given by $Y = X^T w_* + \eps$ for a fixed $w^*\in\reel^{25}$ and $\eps \sim \mathcal{N}(0, 1)$. On Figures~\ref{fig:synth_terms} and~\ref{fig:synth_total} one can see the decomposition of the error between variance and bias terms (Figure \ref{fig:synth_terms}) as well as the sum of both the variance and bias error  (Figure \ref{fig:synth_total}). We see that both the variance and bias curves quickly reach their asymptotic regime with a slope of $-1$ for the variance in log/log space and $-2$ for the bias as expected. We ran the algorithm with two step-sizes, one being ten times larger than the other. We observe the ratio of 100 as expected between the two bias curves and almost no difference at all for the variance ones, except for the first iterations.

One can also see the effect of the step-size $\gamma$ at a fixed number of iterations on Figure~\ref{fig:regime}. Due to the symmetry between $n$ and $\gamma$ in the expression of the bias term, one can notice the resemblance at first between this curve and the one obtained when plotting the $f_n - f^*$ against~$n$.
However, when the step-size is large, we get sooner into the regime where the variance dominates. At this point we observe almost no influence of the step-size on $f_n - f^*$. When getting closer to the maximum step-size, convergence becomes very slow as~$\rho$ becomes close to 1.

\begin{figure}[t]
\centering
\begin{minipage}{.49\textwidth}
\centering
\begin{tikzpicture}
    \begin{loglogaxis}[
        xlabel={Iteration $n$},
        ylabel=$f_n - f^*$,
        ymin=0.5e-8,
        ymax=6,
        x label style={at={(axis description cs:0.5,-0.08)},anchor=north},
        y label style={at={(axis description cs:-0.11,.5)},rotate=0,anchor=south},
        legend entries={{(bias) $\gamma = \gamma_0/10$}, {(variance) $\gamma = \gamma_0/10$}, {(bias) $\gamma =\gamma_0$}, {(variance) $\gamma =\gamma_0$}},
        legend style={legend pos = south west},
        style={font=\footnotesize},
        width=\figwidth]
        \pgfplotstableread{PAPER_hist_dim25_n2_0.1opt_long.txt}\tablesmall
        \pgfplotstableread{PAPER_hist_dim25_n2_1opt_long.txt}\tablelarge
        \addplot[color=blue, mark=o] table[x=n, y=bias] {\tablesmall}
            coordinate [pos=1] (B)
            coordinate [pos=0.5] (D);
        \addplot[color=cyan, mark=+] table[x=n, y=variance] {\tablesmall};
        \addplot[color=red, mark=x] table[x=n, y=bias] {\tablelarge}
            coordinate [pos=1] (A)
            coordinate [pos=0.5] (C)
            ;
        \addplot[color=magenta, mark=*] table[x=n, y=variance] {\tablelarge};

        \draw (A) -- (B);
        \node[anchor=west] at ($ (A)!0.5!(B) $) {($\Delta$)};

        \path[fill=none, draw=none] (B) -- node[sloped, anchor=south]{(slope = -2)} (D);
    \end{loglogaxis}
\end{tikzpicture}
\vspace{-.5\captionsep}
\captionof{figure}{Convergence per term on synthetic data.}
  \vspace{.025\captionsep}
\label{fig:synth_terms}
\end{minipage}
\begin{minipage}{.49\textwidth}
\centering
\begin{tikzpicture}
    \begin{loglogaxis}[
        xlabel={Iteration $n$},
        ylabel=$f_n - f^*$,
        ymin=0.5e-8,
        ymax=6,
        legend entries={{(total) $\gamma = \gamma_0/10$}, {(total) $\gamma =\gamma_0$}},
        legend style={legend pos = south west},
        x label style={at={(axis description cs:0.5,-0.08)},anchor=north},
        y label style={at={(axis description cs:-0.11,.5)},rotate=0,anchor=south},
        style={font=\footnotesize},
        width=\figwidth]
        \pgfplotstableread{PAPER_hist_dim25_n2_0.1opt_long.txt}\tablesmall
        \pgfplotstableread{PAPER_hist_dim25_n2_1opt_long.txt}\tablelarge
        \addplot[color=blue, mark=o] table[x=n, y=total] {\tablesmall}
            coordinate [pos=1] (B)
            coordinate [pos=0.5] (D);
        \addplot[color=red, mark=x] table[x=n, y=total] {\tablelarge}
            coordinate [pos=1] (A)
            coordinate [pos=0.5] (C)
            ;
    \end{loglogaxis}
\end{tikzpicture}
\vspace{-.5\captionsep}
\captionof{figure}{Convergence on synthetic data.}
  \vspace{.25\captionsep}
\label{fig:synth_total}
\end{minipage}
\end{figure}

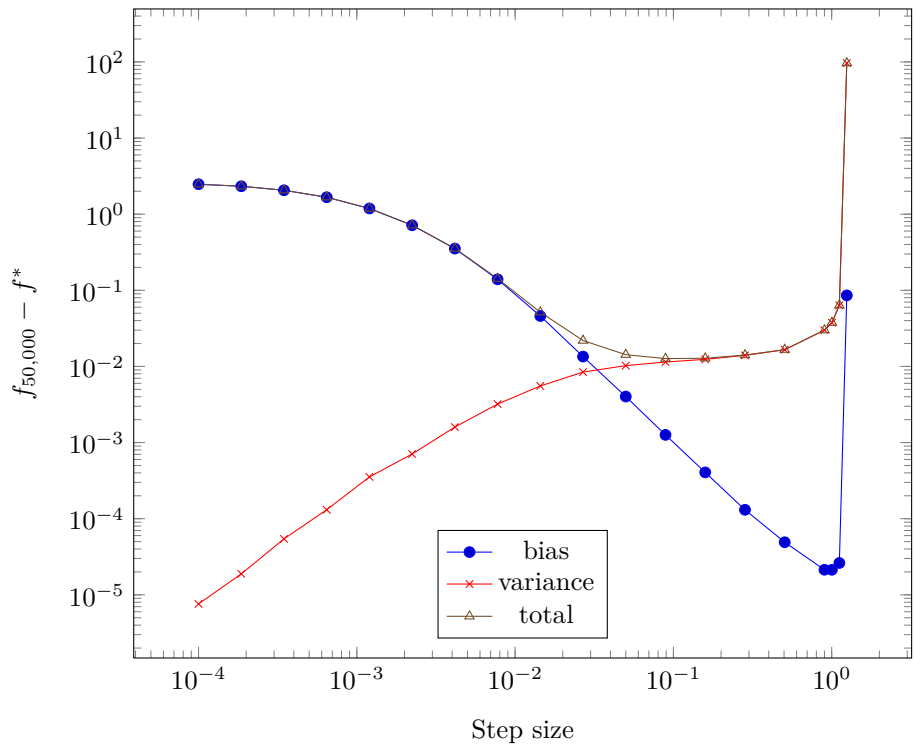
\begin{figure}[t]
    \begin{center}
    \begin{tikzpicture}
        \begin{loglogaxis}[
            xlabel={Step size},
            ylabel={$f_{50,000} - f^*$},
            width=\largefigwidth,
            legend entries={{bias}, {variance}, {total}},
            legend style={at={(0.5,0.03)},anchor=south},
            x label style={at={(axis description cs:0.5,-0.08)},anchor=north},
            y label style={at={(axis description cs:-0.11,.5)},rotate=0,anchor=south},
        ]
        \addplot+[mark=*] table[x=gamma, y=error_bias] {regime_opt_noise_500_sep.txt};
        \addplot+[mark=x] table[x=gamma, y=error_var] {regime_opt_noise_500_sep.txt};
        \addplot+[mark=triangle] table[x=gamma, y=error] {regime_opt_noise_500_sep.txt};

        \end{loglogaxis}
    \end{tikzpicture}
   \vspace{-.95\captionsep}
    \end{center}
    \caption{
    Impact of step-size on error, with its bias/variance decomposition.
    }
      \vspace{-.25\captionsep}
    \label{fig:regime}
\end{figure}

\section{Conclusion}

\vspace*{-.21cm}

In this paper, we have provided a tighter analysis of constant-step-size LMS, leading to  a better understanding of the convergence of the algorithm at different stages, in particular regarding how the initial condition is forgotten.

We were able to deduce different sampling schemes depending on what regime we are in. Sampling proportionally to $\sqrt{X^T H^{-1}X}$ is always asymptotically the best method. The potential gain is however limited most of the time.
 Besides, for datasets that are more ``difficult'', that is with moments that increases quickly, forgetting the initial condition can happen arbitrary slow due to the strong dependency in the step-size. If this is the case, then sampling proportionally to $X^T X$ will allow us to take a much larger step-size which will then lead to a smaller error.

Our work can be extended in several   ways: for simplicity we have focused on least-squares problems where the bias/variance decomposition is explicit. It would be interesting to see how these results can be extended to other smooth losses such as logistic regression, where constant-step size SGD does not converge to the global optimum~\citep{nedic2001convergence,bach13}. Moreover, we have only provided results in expectations and a precise study of higher-order moments would give a better understanding of additional potential effects of resampling.

\section*{Acknowledgements}
  This work was partially supported by a grant from the European Research Council (SIERRA project 239993). We thank Aymeric Dieuleveut and Nicolas Flammarion for interesting discussions related to this work.

\clearpage

\section*{Appendix}
    We give hereafter the proofs for the different results in the main paper. Unless otherwise specified, references are to the present Appendix. We first give a more thorough definition of the space in which our operators live. We then proceed to a proof of Lemma~\ref{lemma:rho}. Finally we detail the computation that allowed us to derive both theorems in this paper.

\ansection{Proof of Lemma 1}

We will first need some preliminary results in order to provide a proof of Lemma 1.

\ansubsection{Some Lemmas}

\begin{lemma}
    \label{lemma_prop_ops_aux}
    Let $A \in \syms$ be any symmetric matrix, then
    \[
    \forall x\in \reel^d, (x^T A x)^2 \leq \Tr{(x^Tx) A x x^T A}.
    \]
\end{lemma}
\begin{proof}
    Using Cauchy-Schwarz inequality, one has
    \begin{align*}
        (x^T A x)^2 = [x^T ( Ax) ]^2 &\leq (A x)^T (A x) (x^T x)\\
                    &= x^T A A x (x^T x) = \Tr{(x^T x) A x x ^T A}.
    \end{align*}
\end{proof}

The following lemma is the proof of equation (2.4) in the original paper.
\begin{lemma}
\label{lemma_n2_opt}
Let $H \in \syms$ be a positive semi-definite matrix.
    If $\gamma >0 $ is so that
    \[
    \forall A\in \syms, \ 2\Tr{A^T H A} - \gamma \esp{(X^T A X)^2} > 0
    \]
    then
    \[\gamma < \frac{2}{\Tr{H}}.\]
\end{lemma}
\begin{proof}
    Let $A \in \syms$,
    $2\Tr{A^T H A} - \gamma \esp{(X^T A X)^2} > 0$ implies that with Jensen's inequality,
    \begin{align*}
        2 \Tr{A^T H A} - \gamma \Tr{A H}^2 &=
             2\Tr{A^T H A} - \gamma \Tr{A \esp{X X^T}}^2\\
             &=2\Tr{A^T H A} - \gamma \esp{X^T A X}^2\\
             &> 0.
    \end{align*}
 Then, let $(u_i)_i \in \reel^{d\times d}$ an orthogonal basis that diagonalizes $H$ and $\lambda_i$ the eigenvalues associated
    with each eigenvector. Then, taking $A = \sum_i u_i u_i^T$, we get
    \begin{align*}
        2 \Tr{A^T H A} - \gamma \Tr{A H}^2 &= 2 \Tr{\sum_{i, j} u_i u_i^T H u_j u_j^T} - \gamma \left(\sum_i u_i^T H u_i\right)^2\\
        &= 2 \left(\sum_i \lambda_i\right) - \gamma \left(\sum_i \lambda_i \right)^2 \geq 0,
    \end{align*}
    so that \[
    \gamma < \frac{2}{\sum_i \lambda_i} = \frac{2}{ \Tr{H} }.
    \]
\end{proof}
\begin{lemma}
    \label{lemma:one_condition}
    Let $\gamma > 0$, we can define $T = H_L + H_R - \gamma M$ as in Section~\ref{sec:algebra}.
    If $\gamma < \frac{2}{\Tr{H}}$, then \[
        I - \gamma T \succ -I.
        \]
    and if we are in dimension 1, \[
        I - \gamma T \succeq 0
    \]
\end{lemma}
\begin{proof}

    The Lemma is equivalent to $\forall A\in \syms, A \neq 0 \Rightarrow \dotp{A, (2 I- \gamma T) A} > 0$.

    If we are in dimension $d=1$, then we have $I -\gamma T = 1 - 2 \gamma h + \gamma m^2$ where $h = \esp{X^2}$ and $m = \esp{X^4} \geq h^2$ so that
    $I - \gamma T \geq (1 - \gamma h)^2 \geq 0$.

    Let now assume we are in dimension two or more.
    Let $A\in \syms$ with $A \neq 0$. Let $P \in \reel^{d \times d}$ be an orthogonal matrix such that $P H P^{-1} = D$ where $D$ is diagonal with eigenvalues ordered in decreasing order, with $\lambda_i = D_{i, i}$ and $\lambda_1 = L$. We will denote $U = P A P^{-1} = PAP^T$.

    \begin{align*}
        \dotp{A, (2I- \gamma T) A}  &=
        \Tr{A^T (2I - \gamma T) A}\\
        &=2\Tr{A^T A} - 2 \gamma\Tr{A^T H A} + \gamma^2 \esp{\left(X^T A X\right)^2}\\
        &\geq 2\Tr{A^T A} - 2 \gamma\Tr{A^T H A} + \gamma^2 \esp{\left(X^T A X\right)}^2\\
        &= 2\Tr{A^T A} - 2 \gamma\Tr{A^T H A} + \gamma^2 \Tr{A H}^2\\
        &= 2\Tr{U^T U} - 2 \gamma\Tr{U^T D U} + \gamma^2 \Tr{U D}^2\\
        &= \sum_{i, j=1}^d 2 U_{i,j}^2 - 2 \gamma U_{i, j}^2 \lambda_i + \gamma^2 U_{i, i} U_{j, j} \lambda_{i} \lambda_{j}\\
        &= \left(\sum_{i \neq j} 2 U_{i, j}^2 (2 - \gamma (\lambda_i + \lambda_j))
        \right)
        + \sum_{i=1}^d 2 U_{i, i}^2 - 2 \gamma U_{i, i}^2 \lambda_i
        + \gamma^2 \left(\sum_{i=1}^d U_{i, i} \lambda_i\right)^2.
                                    \end{align*}

    The first sum immediately defines a definite positive form over the subspace generated by $(U_{i, j})_{i\neq j}$ as $\gamma < \frac{2}{\lambda_i + \lambda_j}$ for all $i \neq j$. The second part also defines a bilinear form over the orthogonal subspace generated by $(U_{i, i})_{1 \leq i\leq d}$. $2 I - \gamma T$ is definite positive if and only if those two forms are definite positive.
    We will introduce $x_i = U_{i, i}$ so that the second form is given by $x^T G x$ where
    $G = 2 I - 2 \gamma\mathrm{Diag}(\Lambda) + \gamma^2 \Lambda \Lambda^T$, with $\Lambda = (\lambda_i)_{1\leq i\leq d}$ and $\mathrm{Diag}(\Lambda)$ the diagonal matrix with values from $\Lambda$ on the diagonal.

    We can decompose G as
    \begin{align*}
        G &= \begin{pmatrix}
            B & \gamma^2 \lambda_1 C^T \\
            \gamma^2 \lambda_1 C & D
        \end{pmatrix},
    \end{align*}
    with $B = 2 - 2 \gamma \lambda_1 + \gamma^2 \lambda_1$,
    $C = (\lambda_i)_{2\leq i\leq d}$ and $D = 2 I - 2 \gamma \mathrm{Diag}(C) + \gamma^2 C C^T$. Using the Schur completement condition for positive definiteness, we have that $G \succ 0$ if and only if $D \succ 0$ and $B - \gamma^4 \lambda_1^2 C^T D^{-1} C > 0$. We immediately have that $D \succ 0$ as $I - \gamma \mathrm{Diag}(C) \succ 0$, indeed, for all $d \geq i \geq 2$, we have that $\gamma \lambda_i < 1$.

    Let us introduce $E = 2 I - 2\gamma \mathrm{Diag}(C)$, then we have
    \begin{align*}
        D^{-1} &=  E^{-1} - \frac{\gamma^2}{1 + \gamma^2 C^T E^{-1} C} E^{-1} C C^T E^{-1}.
    \end{align*}
We will assume that $\sum_{i =2}^d \lambda_i < \lambda_1$, otherwise one trivially has that $\gamma \lambda_1 < 1$ and $G \succ 0$.
    Let us denote
    \begin{align*}
        q &= C^T E^{-1} C\\
        &= \sum_{i=2}^d \frac{\lambda_i^2}{2 (1 - \lambda_i \gamma)}\\
        &\leq \frac{(\sum_{i=2}^d \lambda_i)^2}{2 (1 - \gamma \sum_{i=2}^d \lambda_i)}\\
        &= \frac{l^2}{2 (1 - \gamma l)},
    \end{align*}
    where $l = \sum_{i=2}^d \lambda_i$. We will take $l = \lambda_1 \alpha$ so that $0 < \alpha < 1$. We have
    \begin{align*}
        B - \gamma^4 \lambda_1^2 C^T D^{-1} C &=
        \gamma^2\lambda_1^2 + 2 - 2 \lambda_1 \gamma
        -\gamma^4 \lambda_1^2 \left(q - \frac{\gamma^2 q^2}{1+\gamma^2 q}\right)\\
        &=\frac{ \gamma^2\lambda_1^2}{1 +\gamma^2 q} - 2 \lambda_1\gamma + 2\\
        &\geq \frac{\gamma^2 \lambda_1^2}{1+\gamma^2 \frac{l^2}{2 (1 - \gamma l)}}
            - 2 \lambda_1\gamma + 2.
    \end{align*}
    Denoting $y = \gamma \lambda_1$, we get
    \begin{align*}
        B - \gamma^4 \lambda_1^2 C^T D^{-1} C &=
        \frac{2y^2 (1 - y \alpha)}{2 - 2 y\alpha + \alpha^2 y^2}
        - 2 y + 2.
    \end{align*}
    Using standard analysis tools, one can show that the last quantity is positive for $0 < y < \frac{2}{1+\alpha}$ and $0 < \alpha < 1$.
    As a conclusion, $G$ is definite positive and so is $2 I - \gamma T$.

            \end{proof}

\begin{lemma}
\label{lemma_prop_ops}
Let $\gamma > 0$, we can define $T = H_L + H_R - \gamma M$ which is symmetric and is stable over $\syms$.

 If
\begin{align*}
    \forall A\in \syms, \,2 \Tr{A^T H A} - \gamma \esp{(X^TA X)^2} > 0,
\end{align*}
or (this second assumption implies the first one)
 \begin{align*}
        \esp{X X^T} - \gamma \esp{X^T X X X^T} \succ 0,
    \end{align*}

    then
    \begin{itemize}
        \item $\normop{I - \gamma H} < 1$ ,
        \item $T \succ 0$ ,
        \item $\normop{I-\gamma T}  < 1$.
    \end{itemize}
\end{lemma}
\begin{proof}
    We should first notice that using Lemma~\ref{lemma_n2_opt}, we necessarely have \begin{equation}
    \label{eq:gamma_bound}
    \gamma < \frac{2}{\Tr{H}}.
    \end{equation}

    We first need, $I - \gamma H \prec I$ which is always true as long as $H$ is invertible  (i.e. $H$ is positive).
    Then we need $I - \gamma H \succ -I$, or $\gamma H \prec 2 I$, which means $\gamma < \frac{2}{L}$ where $L$ is $H$ largest eigenvalue. However this is implied by \eqref{eq:gamma_bound}.

    Now, we need $I - \gamma T \prec I$, i.e., $T \succ 0$ (this will also prove $T$ invertible). This is equivalent to
    \begin{align*}
         \forall A \in \syms, A\neq 0 \Rightarrow \dotp{A, T A} > 0
    \end{align*}

Let us compute this term for $A \in \syms$ with $A \neq 0$
    \begin{align*}
        \dotp{A, T A} &= \Tr{A^T (T A)} \\&=  \Tr{A^T A H + A^T H A -\gamma A^T\esp{X X^T A X X^T}}\\
        &= 2\Tr{A^T H A} - \esp{(X^T A X)^2} \quad\text{and we can stop here if we have first assumption}\\
            &\geq
            \Tr{A^T \left(2 H - \gamma \esp{X^T X X X^T}\right) A } \quad \text{using Lemma \ref{lemma_prop_ops_aux}}
    \end{align*}

    A sufficient condition here is that  $K = 2 H - \gamma \esp{X^T X X X^T} \succ 0$. Indeed, let $I = \mathrm{Ker}(A)^\bot$ be the orthogonal space of the kernel of $A$, which is stable under $A$ as $A$ is symmetric, so we can define $A'$ the restriction of $A$ to $I$ which is invertible. It is of dimension greater than 1 as $A$ is not 0. $K$ defines on $I$ a bilinear symmetric definite positive application $K'$. Then, $\Tr{A^T K A} = \Tr{A'^T K' A'} > 0$ because $A'^T K' A'$ is also symmetric definite positive.

Finally, we want $I - \gamma T \succ -I$. Using Lemma \ref{lemma:one_condition}, this is a direct consequence of \eqref{eq:gamma_bound}.
\end{proof}

\ansubsection{Proof of Lemma \ref{lemma:rho}}

Let assume $0 < \gamma < \gmax$.
Lemma~\ref{lemma_prop_ops} already tells us that our operators have good properties as we have $\rho < 1$ and $T \succ 0$.
We will now get a finer result in order to have an explicit bound on $\rho$ depending on $\gamma$.

As we will be using different values for $\gamma$ we will explicitely mark the dependency in $\gamma$ for $T$ by writing $T(\gamma)$. We will only consider $0 < \gamma < \gmax$ so that $T(\gamma)$ is positive.
We will denote by $L_{T(\gamma)}$ the largest eigenvalue of $T(\gamma)$ and by $\mu_{T(\gamma)}$ its smallest. We then have
\begin{align*}
    \rho_T(\gamma) = \max(1 - \gamma \mu_{T(\gamma)}, \gamma L_{T(\gamma)} - 1).
\end{align*}
One should also notice that the smallest eigenvalue of $H_L + H_R$ is $2\mu$ and the largest $2 L$.

We have $T(\gmax) \succeq 0$ using Lemma~\ref{lemma_prop_ops}. For any $0<\gamma < \gmax$ we can define $\alpha = \frac{\gamma}{\gmax}$. Then we have
\begin{align*}
    T(\gamma) &= (1 - \alpha) (H_L + H_R) + \alpha (H_L + H_R) - \alpha \gmax M\\
              &= (1 - \alpha) (H_L + H_R) + \alpha T(\gmax)\\
              &\succeq (1 - \alpha) (H_L + H_R)\\
              &\succeq 2 (1 - \alpha) \mu,
\end{align*}
so that $\mu_{T(\gamma)} \geq 2 (1 - \alpha) \mu$.

Using Lemma~\ref{lemma:one_condition} we have that $T(\gmax) \preceq \frac{2 I}{\gmax}$ so that we obtain
\begin{align*}
    T(\gamma) &= (1 -\alpha) (H_L + H_R) + \alpha T(\gmax) \\
    &\preceq 2 (1 - \alpha) L + \frac{2\alpha}{\gmax}.
\end{align*}

As a consequence if we take
\begin{align*}
    a(\gamma) = 1 - 2 \alpha \gmax (1 - \alpha)\mu\\
    b(\gamma) = 2 (1 - \alpha) \alpha \gmax L + 2\alpha^2 - 1
\end{align*}
we have $\rho_T(\gamma) = \max(a(\gamma), b(\gamma))$. Besides, if we are in dimension $d=2$ or more,
\begin{align*}
    a(\gamma) - b(\gamma) &=
    2 - 2\alpha \gmax (L + \mu) (1 - \alpha) - 2 \alpha^2\\
    &\geq 2 - 4\alpha (1 - \alpha) - 2 \alpha^2 \quad\text{as $\gmax (L + \mu) \leq 2$}\\
    &= 2 + 2 \alpha^2  - 4 \alpha\\
    &\geq 0,
\end{align*}
so that,
\[
    \rho_T(\gamma) \leq 1 - 2 \gamma \mu(1 - \frac{\gamma}{\gmax}).
\]

In dimension $d=1$, the same result holds as we have $1 - \gamma T \geq 1 - 2 \gamma H + \gamma^2 H^2 \geq 0$ so that $\gamma L_{T(\gamma)} - 1 \leq 0$.

We can now look as $\rho_H$ which is given by $\rho_H = \max(1 - \gamma \mu, \gamma L - 1)$.

Let assume we are in dimension 2 or more, then we have $1 - \gamma \mu \geq \gamma L - 1$ so that $\rho_H = 1 - \gamma \mu$.
In dimension 1, we have $\rho_H = \abs{1 - \gamma\mu}$. Comparing $\rho_H$ and $\rho_T$ we obtain the result of this Lemma.

Finally, if $\gamma > \gmax$, then $T$ has a negative eigenvalue and so $\rho_T > 1$ and $\rho > 1$. \note{FB: is this really true?} \noteal{It is true that $\rho > 1$ as $I-\gamma T$ will have one eigen value strictly larger than 1. however it does not necessarely prove divergence.}

\ansection{Proof of the theorems}

\ansubsection{Complete expression of the covariance matrix}

Let us recall that we have the update rule
\begin{align}
\label{eq:update}
    \eta_i = (I - \gamma X_i X_i^T) \eta_i + \gamma \eps_i X_i.
\end{align}

We can then introduce the following matrices
\begin{align*}
    M_{k, j} = \left(\prod_{i=k+1}^{j} \left(I - \gamma X_{i} X_i^T \right)\right)^T \in \reel^{d \times d},
\end{align*}

and by iterating over \eqref{eq:update} we obtain,
\begin{align*}
    \eta_{n} = \gamma \sum_{k=1}^{n} M_{k,n} X_k \eps_k + M_{0,n} \eta_0.
\end{align*}

We have
\begin{align*}
    \bar{\eta}_n &= \frac{\gamma}{n} \sum_{j=0}^{n-1} \sum_{k=1}^j M_{k,j}X_k \eps_k + \frac{1}{n}\sum_{j=0}^{n-1} M_{0, j} \eta_0\\
                 &= \frac{\gamma}{n} \sum_{k=1}^{n-1} \left(
                    \sum_{j=k}^{n-1} M_{k, j}
                 \right) X_k \eps_k + \frac{1}{n}\sum_{j=0}^{n-1} M_{0, j} \eta_0.
\end{align*}
One can already see the decomposition between the variance and bias term, one depending only on $\eta_0$ and the other on $\eps$.

If we assume that $\eps_k$ is independent of $X_k$, then we can immediately see that when computing $\esp{\bar{\eta}_n\bar{\eta}_n^T}$, cross-terms between bias and variance will be zero as they will contain only one $\eps_k$.
If that is not true,
then extra cross-terms will appear and there is no longer a simple bias/variance decomposition.
Let us look at one of the cross terms,
\begin{align*}
    \frac{\gamma}{n^2} \esp{M_{k,j} X_k \eps_k \eta_0^T M_{0,p}}.
\end{align*}
If $p < k$, then one can immediately notice that $X_k \eps_k$ will be independant from the rest so that the term will be 0, as it is always true that $\esp{\eps X} = 0$. If not, $X_k$ will also appear in $M_{0,p}$ as a factor $I - \gamma X_k X_k^T$ so that the term can be expressed as $G(\esp{X_k \eps_k \eta_0^T X_k X_k^T})$ where $G$ is a linear operator obtained using the independance of the other $X_i$ and $\eps_i$ for $i \neq k$. As a consequence, we can recover a simple decomposition as soon as
\begin{align*}
    \forall \,1 \leq i, j, k \leq d, \esp{X^{(i)} X^{(j)} X^{(k)} \eps} = 0,
\end{align*}
where $X^{(i)}$ is the i-th component of $X$.

In any case, because of Minkowski's inequality as noted in \cite{bach13}, we always have that
\begin{align*}
    f^{\mathrm{total}}_n - f^* \leq 2 (f^{\mathrm{bias}}_n - f^*) + 2 (f^{\mathrm{variance}}_n - f^*),
\end{align*}
so that we are never too far from the true error when assuming $X$ and $\eps$ independant.

\ansubsection{Proof for the bias term}
 First, let us assume that $\eps_k = 0 \as$.
 Then we have
 \begin{align*}
     \bar{\eta}_n  &= \frac{1}{n}\sum_{j=0}^{n-1} M_{0, j} \eta_0,
 \end{align*}
 and
 \begin{align*}
     \esp{\bar{\eta}_n \bar{\eta}_n^T} = &
        \frac{1}{n^2}\sum_{i=0}^{n-1}\sum_{j=0}^{n-1} \esp{M_{0, i} \eta_0 \eta_0^T M_{0, j}^T}\\
        = &\frac{1}{n^2}\sum_{i=0}^{n-1}\left(
            \esp{M_{0, i} \eta_0 \eta_0^T M_{0, i}^T
            + \sum_{j = i+1}^{n-1} M_{0, i} \eta_0 \eta_0^T M_{0, i}^T M_{i, j}^T
            + \sum_{j = 0}^{i-1} M_{j, i} M_{0, j} \eta_0 \eta_0^T M_{0, j}^T}
        \right)\\
        = &\frac{1}{n^2}\sum_{i=0}^{n-1}\Biggl(
            \esp{M_{0, i} \eta_0 \eta_0^T M_{0, i}^T}
            + \sum_{j = i+1}^{n-1} \esp{M_{0, i} \eta_0 \eta_0^T M_{0, i}^T} (I - \gamma H)^{j - i}\\
            &{}+ \sum_{j = 0}^{i-1} (I - \gamma H)^{i - j} \esp{M_{0, j}^T \eta_0 \eta_0^T M_{0, j}}
        \Biggr) \mbox{ because of independence assumptions,}\\
        = &\frac{1}{n^2}\sum_{i=0}^{n-1}\left(
            \esp{M_{0, i} \eta_0 \eta_0^T M_{0, i}^T}
            + \sum_{j = i+1}^{n-1} \esp{M_{0, i} \eta_0 \eta_0^T M_{0, i}^T} (I - \gamma H)^{j - i}
            \right)\\
        &{}+
            \frac{1}{n^2}\sum_{j=0}^{n-1}
            \left(\sum_{i = j+1}^{n-1} (I - \gamma H)^{i - j} \esp{M_{0, j}^T \eta_0 \eta_0^T M_{0, j}}
        \right)\\
        = &\frac{1}{n^2}\sum_{i=0}^{n-1}\Biggl(
            \esp{M_{0, i} \eta_0 \eta_0^T M_{0, i}^T}\\
            &{}+ \sum_{j = i+1}^{n-1} \left(\esp{M_{0, i} \eta_0 \eta_0^T M_{0, i}^T} (I - \gamma H)^{j - i} +
                                                (I - \gamma H)^{j - i} \esp{M_{0, i} \eta_0 \eta_0^T M_{0, i}^T}\right)
        \Biggr)\\
        & \mbox{ by exchanging the role of } i \mbox{ and } j \mbox{ in the last equation,} \\
        = &\frac{1}{n^2}\sum_{i=0}^{n-1} \bigg(
            \esp{M_{0, i} \eta_0 \eta_0^T M_{0, i}^T}\\
           &{}+ \esp{M_{0, i} \eta_0 \eta_0^T M_{0, i}^T} \left((I - \gamma H) - (I - \gamma H)^{n-i}\right)(\gamma H)^{-1}\\
           &{}+ (\gamma H)^{-1}\left((I - \gamma H) - (I - \gamma H)^{n-i}\right)\esp{M_{0, i} \eta_0 \eta_0^T M_{0, i}^T} \bigg).
 \end{align*}

We only used the fact that $X_i$ and $X_j$ are independent as soon as $i \neq j$, so that we can condition on $X_1, \ldots X_i$
to obtain $M_{1, i} (I - \gamma H)^{j -i}$.
Now we need to express $\esp{(I - \gamma X_i X_i^T)A (I - \gamma X_i X_i^T)}$ for $A$ some matrix that is independent of $X_i$. Using the notation we introduced, we have immediately that
\begin{align*}
    \esp{(I - \gamma X_i X_i^T)A (I - \gamma X_i X_i^T)} &=
        A - \gamma A H - \gamma H A + \gamma^2 \esp{X^T A X X X^T}\\
        &= (I - \gamma H_R - \gamma H_L + \gamma^2 M) A\\
        &= (I - \gamma T) A.
\end{align*}

Then we have, with $\mathcal{F}_{i-1}$ the $\sigma$ field generated by $X_1, \ldots, X_{i-1}$,
\begin{align*}
    \esp{M_{0, i} \eta_0 \eta_0^T M_{0, i}^T} &=
        \esp{\esp{M_{0, i} \eta_0 \eta_0^T M_{0, i}^T| \mathcal{F}_{i-1}}}\\
        &=
            \esp{\esp{(I - \gamma X_i X_i^T) M_{0, i-1} \eta_0 \eta_0^T M_{0, i-1}^T (I - \gamma X_i X_i^T) | \mathcal{F}_{i-1}}}\\
        &= \esp{(I - \gamma T) M_{0, i-1} \eta_0 \eta_0^T M_{0, i-1}}\\
        &= (I-\gamma T)\esp{M_{0, i-1} \eta_0 \eta_0^T M_{0, i-1}}.
\end{align*}

and by iterating this process, we obtain
\begin{align*}
    \esp{\bar{\eta}_n \bar{\eta}_n^T} =
    \frac{1}{n^2}&\sum_{i=0}^{n-1}
            (I - \gamma T)^{i} E_0\\
           &{}+ \left((I - \gamma T)^{i}E_0\right) \left((I - \gamma H) - (I - \gamma H)^{n-i}\right)(\gamma H)^{-1}\\
           &{}+ (\gamma H)^{-1}\left((I - \gamma H) - (I - \gamma H)^{n-i}\right)\left((I - \gamma T)^{i}E_0\right)\\
    = \frac{1}{n^2}&\sum_{i=0}^{n-1}
        \biggl(I + \left[(I-\gamma H)_L - (I -\gamma H)_L^{n-i}\right](\gamma H_L)^{-1} \\
           &{}+ \left[(I-\gamma H)_R - (I -\gamma H)_R^{n-i}\right](\gamma H_R)^{-1}
           \biggr)(I - \gamma T)^{i} E_0.
\end{align*}

Let us define
\begin{align*}
    A_n &= -\frac{1}{n^2} \sum_{i=0}^{n-1} \left(
    (\gamma H_R)^{-1}(I - \gamma H)_R^{n-i}+
        (\gamma H_L)^{-1}(I - \gamma H)_L^{n-i}
    \right)\left((I - \gamma T)^{i}E_0\right)\\
    \norm{A_n}_F &\leq
        \frac{2 d}{n \gamma \mu} \rho^n \norm{E_0}_F,
\end{align*}
which is decaying exponentially. \note{FB: a bit more steps there} \noteal{For bounding $A_N$ or for what follows?} We now have
\begin{align*}
    \esp{\bar{\eta}_n \bar{\eta}_n^T} &=
        \frac{1}{n^2}\sum_{i=0}^{n-1}
                \left(
                    I + (I-\gamma H_L)(\gamma H_L)^{-1} + (I-\gamma H_R)(\gamma H_R)^{-1}
                \right)(I - \gamma T)^{i} E_0 + A_n\\
    &=
    \frac{1}{\gamma^2 n^2}
            \left(H_L^{-1} + H_R^{-1}- \gamma I\right)T^{-1}\left(I - (I - \gamma T)^{n}\right) E_0
            +A_n.
\end{align*}

Again, we have some exponential terms, that we will regroup in $B_n$ with
\begin{align*}
    B_n &=-\frac{1}{\gamma^2 n^2} \left(H_L^{-1} + H_R^{-1}- \gamma I\right)T^{-1}(I - \gamma T)^{n}E_0\\
    \norm{B_n}_F &\leq \frac{d}{n^2 \gamma^2\mu_T}\rho_T^n \left(
        \frac{2}{\mu} - \gamma
    \right)\norm{E_0}_F,
\end{align*}
and we have
\begin{align*}
    \esp{\bar{\eta}_n \bar{\eta}_n^T} &=
        \frac{1}{n^2\gamma^2}\left(H_L^{-1} + H_R^{-1}- \gamma I\right)T^{-1}E_0
            +A_n + B_n.
\end{align*}
We can bound $A_n + B_n$ by
\begin{align*}
    \norm{A_n + B_n}_F \leq \frac{d \rho^n \norm{E_0}_F}{\gamma n}\left(
        \frac{2}{\mu} + \frac{1}{\mu_T n \gamma} \left(\frac{2}{\mu} - \gamma\right)
    \right),
\end{align*}
which completes the first assertion of Theorem~\ref{bias_theo}.

\ansubsection{Proof for the variance term}
Let assume now that $\eta_0 = 0$, then we have
\begin{align*}
    \bar{\eta}_n &= \frac{\gamma}{n} \sum_{k=1}^{n-1} \left(
                    \sum_{j=k}^{n-1} M_{k, j}
                 \right) X_k \eps_k ,
\end{align*}
and
\begin{align*}
\esp{\bar{\eta}_n \bar{\eta}_n^T} &=\frac{\gamma^2}{n^2}
        \esp{\sum_{k,l=1}^{n-1}\left( \sum_{j=k}^{n-1}M_{k, j}\right) X_k \eps_k
            \eps_l X_l^T \left(\sum_{p=l}^{n-1} M_{l, p}^T\right)}\\
    &=\frac{\gamma^2}{n^2}
        \esp{\sum_{k=1}^{n-1}\left( \sum_{j=k}^{n-1}M_{k, j}\right) X_k \eps_k
            \eps_k X_k^T \left(\sum_{p=k}^{n-1} M_{k, p}^T\right)}.
\end{align*}

Indeed, we can remove terms where $k \neq l$: if we have for instance $l < k$, then $X_l \eps_l$ will be independent from the rest of the terms and as $\esp{X_l \eps_l} = 0$, the term will be 0.

By using mostly the same method as for the bias term, we obtain that
\begin{align*}
    \esp{\bar{\eta}_n \bar{\eta}_n^T} =
    &\frac{\gamma^2}{n^2}
            \sum_{k=1}^{n-1} \sum_{j=k}^{n-1} \left(I - \gamma T\right)^{j-k}\Sigma_0\\
            &{} + \left((I - \gamma H) - (I - \gamma H)^{n - j}\right)(\gamma H)^{-1} \left(\left(I - \gamma T\right)^{j-k}\Sigma_0\right)\\
            &{} + \left(\left(I - \gamma T\right)^{j-k}\Sigma_0\right)\left((I - \gamma H) - (I - \gamma H)^{n - j}\right)(\gamma H)^{-1} \\
        =   &\frac{\gamma^2}{n^2}
            \sum_{j=1}^{n-1} \sum_{k=1}^{j} \left(I - \gamma T\right)^{j-k}\Sigma_0\\
            &{} + \left((I - \gamma H) - (I - \gamma H)^{n - j}\right)(\gamma H)^{-1} \left(\left(I - \gamma T\right)^{j-k}\Sigma_0\right)\\
            &{} + \left(\left(I - \gamma T\right)^{j-k}\Sigma_0\right)\left((I - \gamma H) - (I - \gamma H)^{n - j}\right)(\gamma H)^{-1} \\
        = &\frac{\gamma^2}{n^2}\sum_{j=1}^{n-1}
        \left(I - (I - \gamma T)^j\right) (\gamma T)^{-1} \Sigma_0\\
         &+\left((I - \gamma H) - (I - \gamma H)^{n - j}\right)(\gamma H)^{-1} \left(I - (I - \gamma T)^j\right) (\gamma T)^{-1} \Sigma_0\\
        &+ \left(I - (I - \gamma T)^j\right) (\gamma T)^{-1} \Sigma_0
        \left((I - \gamma H) - (I - \gamma H)^{n - j}\right)(\gamma H)^{-1}.
\end{align*}

As for the bias, we can bound some terms:
\begin{align*}
    C_n &=\frac{\gamma^2}{n^2}\sum_{j=1}^{n-1}\left(
    (I - \gamma H)_L^{n - j}(\gamma H_L)^{-1}+
    (I - \gamma H)_R^{n - j}(\gamma H_R)^{-1}
    \right) (I - \gamma T)^j (\gamma T)^{-1} \Sigma_0 \\
    \norm{C_n}_F&\leq \frac{2 d}{n \mu \mu_T} \rho^n \norm{\Sigma_0}_F.
\end{align*}
Now we have,
\begin{align*}
    \esp{\bar{\eta}_n \bar{\eta}_n^T} =
        &
        \frac{1}{n^2}
        \sum_{j=1}^{n-1}
        \left(
            H_L^{-1} + H_R^{-1} - \gamma I
        \right) \left(I - (I - \gamma T)^j\right)
        T^{-1} \Sigma_0 +C_n\\
    = &
    \frac{1}{n}
        \left(
            H_L^{-1} + H_R^{-1} - \gamma I
        \right)
        T^{-1} \Sigma_0
    + D_n
     + C_n,
\end{align*}
where $D_n$ is defined by
\begin{align*}
    D_n = &-\frac{1}{n^2}
        \sum_{j=1}^{n-1}
        \left(
            H_L^{-1} + H_R^{-1} - \gamma I
        \right) (I - \gamma T)^j
        T^{-1} \Sigma_0\\
    = &-\frac{1}{\gamma n^2} \left(
            H_L^{-1} + H_R^{-1} - \gamma I
        \right) (I - \gamma T) T^{-2} \Sigma_0
        + D'_n.
\end{align*}
$D'_n$ are again exponentially decreasing terms:
\begin{align*}
    D'_n =  &\frac{1}{\gamma n^2} \left(
            H_L^{-1} + H_R^{-1} - \gamma I
        \right) (I - \gamma T)^n T^{-2} \Sigma_0\\
    \norm{D'_n}_F &\leq \frac{d}{\gamma^2\mu_T^2 n}\left(\frac{2}{\mu} - \gamma\right) \rho_T^n \norm{\Sigma_0}_F,
\end{align*}
so that we have
\begin{align}
\label{eq:full_expr_var}
    \esp{\bar{\eta}_n \bar{\eta}_n^T} &=
        \frac{1}{n}
            \left(
                H_L^{-1} + H_R^{-1} - \gamma I
            \right)
            T^{-1} \Sigma_0 -\frac{1}{\gamma n^2} \left(
            H_L^{-1} + H_R^{-1} - \gamma I
        \right) (I - \gamma T) T^{-2} \Sigma_0
        + C_n + D'_n.
\end{align}
We can bound $C_n + D'_n$ by
\begin{align*}
    \norm{C_n + D'_n}_F \leq \frac{d \rho^n \norm{\Sigma_0}_F}{n}\left(
        \frac{1}{n \gamma \mu_T^2} \left(\frac{2}{\mu} - \gamma\right) +
        \frac{2}{\mu \mu_T}
    \right).
\end{align*}
This concludes the proof of Theorem~\ref{var_theo}.

\clearpage

    \bibliography{biblio}

\begin{thebibliography}{}

\bibitem[\protect\citeauthoryear{Bach and Moulines}{Bach and
  Moulines}{2011}]{bach11}
Bach, F. and E.~Moulines (2011).
\newblock {Non-Asymptotic Analysis of Stochastic Approximation Algorithms for
  Machine Learning}.
\newblock In {\em Adv. NIPS}.

\bibitem[\protect\citeauthoryear{Bach and Moulines}{Bach and
  Moulines}{2013}]{bach13}
Bach, F. and E.~Moulines (2013).
\newblock Non-strongly-convex smooth stochastic approximation with convergence
  rate \textsc{O}(1/n).
\newblock In {\em Adv. NIPS}.

\bibitem[\protect\citeauthoryear{Bershad}{Bershad}{1986}]{normalized}
Bershad, N. (1986).
\newblock Analysis of the normalized \textsc{LMS} algorithm with gaussian
  inputs.
\newblock {\em Speech and Signal Processing, IEEE Transactions on
  Acoustics\/}~{\em 34\/}(4), 793--806.

\bibitem[\protect\citeauthoryear{Bottou and Le~Cun}{Bottou and
  Le~Cun}{2005}]{ASMB:ASMB538}
Bottou, L. and Y.~Le~Cun (2005).
\newblock On-line learning for very large data sets.
\newblock {\em Applied Stochastic Models in Business and Industry\/}~{\em
  21\/}(2), 137--151.

\bibitem[\protect\citeauthoryear{Bousquet and Bottou}{Bousquet and
  Bottou}{2008}]{NIPS2007_3323}
Bousquet, O. and L.~Bottou (2008).
\newblock The tradeoffs of large scale learning.
\newblock In {\em Adv. NIPS}.

\bibitem[\protect\citeauthoryear{Fabian}{Fabian}{1968}]{fabian1968}
Fabian, V. (1968).
\newblock On asymptotic normality in stochastic approximation.
\newblock {\em The Annals of Mathematical Statistics\/}~{\em 39\/}(4),
  1327--1332.

\bibitem[\protect\citeauthoryear{Kanamori and Shimodaira}{Kanamori and
  Shimodaira}{2003}]{KANAMORI}
Kanamori, T. and H.~Shimodaira (2003).
\newblock {Active learning algorithm using the maximum weighted log-likelihood
  estimator}.
\newblock {\em Journal of statistical planning and inference\/}~{\em 116\/}(1),
  149--162.

\bibitem[\protect\citeauthoryear{Macchi}{Macchi}{1995}]{macchi1995adaptive}
Macchi, O. (1995).
\newblock {\em Adaptive processing: The least mean squares approach with
  applications in transmission}.
\newblock Wiley West Sussex.

\bibitem[\protect\citeauthoryear{Nedic and Bertsekas}{Nedic and
  Bertsekas}{2000}]{nedic2001convergence}
Nedic, A. and D.~Bertsekas (2000).
\newblock {Convergence rate of incremental subgradient algorithms}.
\newblock {\em Stochastic Optimization: Algorithms and Applications\/},
  263--304.

\bibitem[\protect\citeauthoryear{Needell, Srebro, and Ward}{Needell
  et~al.}{2013}]{needell2013stochastic}
Needell, D., N.~Srebro, and R.~Ward (2013).
\newblock Stochastic gradient descent and the randomized kaczmarz algorithm.
\newblock Technical Report 1310.5715, arXiv.

\bibitem[\protect\citeauthoryear{Nemirovski and Yudin}{Nemirovski and
  Yudin}{1983}]{nemirovsky1983problem}
Nemirovski, A.~S. and D.~B. Yudin (1983).
\newblock {\em Problem complexity and method efficiency in optimization.}
\newblock Wiley \& Sons.

\bibitem[\protect\citeauthoryear{Nesterov}{Nesterov}{2004}]{nesterov2004introductory}
Nesterov, Y. (2004).
\newblock {\em Introductory Lectures on Convex Optimization: a Basic Course}.
\newblock Kluwer Academic Publishers.

\bibitem[\protect\citeauthoryear{Nesterov}{Nesterov}{2012}]{nesterov2012efficiency}
Nesterov, Y. (2012).
\newblock Efficiency of coordinate descent methods on huge-scale optimization
  problems.
\newblock {\em SIAM Journal on Optimization\/}~{\em 22\/}(2), 341--362.

\bibitem[\protect\citeauthoryear{Perronnin, Akata, Harchaoui, and
  Schmid}{Perronnin et~al.}{2012}]{perronnin2012towards}
Perronnin, F., Z.~Akata, Z.~Harchaoui, and C.~Schmid (2012).
\newblock Towards good practice in large-scale learning for image
  classification.
\newblock In {\em Proc. CVPR}.

\bibitem[\protect\citeauthoryear{Polyak and Juditsky}{Polyak and
  Juditsky}{1992}]{Polyak}
Polyak, B.~T. and A.~B. Juditsky (1992, July).
\newblock Acceleration of stochastic approximation by averaging.
\newblock {\em SIAM J. Control Optim.\/}~{\em 30\/}(4), 838--855.

\bibitem[\protect\citeauthoryear{Ruppert}{Ruppert}{1988}]{ruppert}
Ruppert, D. (1988).
\newblock Efficient estimations from a slowly convergent {R}obbins-{M}onro
  process.
\newblock Technical report, Cornell University Operations Research and
  Industrial Engineering.

\bibitem[\protect\citeauthoryear{Schmidt, Roux, and Bach}{Schmidt
  et~al.}{2013}]{sag}
Schmidt, M., N.~L. Roux, and F.~Bach (2013).
\newblock Minimizing finite sums with the stochastic average gradient.
\newblock Technical Report 00860051, HAL.

\bibitem[\protect\citeauthoryear{Shalev-Shwartz and Zhang}{Shalev-Shwartz and
  Zhang}{2013}]{sdca}
Shalev-Shwartz, S. and T.~Zhang (2013).
\newblock Stochastic dual coordinate ascent methods for regularized loss
  minimization.
\newblock {\em JMLR\/}~{\em 14}, 567---599.

\bibitem[\protect\citeauthoryear{Toulis, Rennie, and Airoldi}{Toulis
  et~al.}{2014}]{implicit}
Toulis, P., J.~Rennie, and A.~M. Airoldi (2014).
\newblock Statistical analysis of stochastic gradient methods for generalized
  linear models.
\newblock In {\em Proc. ICML}.

\bibitem[\protect\citeauthoryear{Zhao and Zhang}{Zhao and
  Zhang}{2014}]{Zhao2014}
Zhao, P. and T.~Zhang (2014).
\newblock Stochastic optimization with importance sampling.
\newblock Technical Report 1401.2753, arXiv.

\end{thebibliography}
\end{document}